\documentclass[11pt]{article}

\pdfoutput=1

\usepackage{natbib}

\usepackage[T1]{fontenc}
\usepackage[latin9]{inputenc}
\usepackage{float}
\usepackage{fullpage,times,color}
\usepackage{boxedminipage}

\usepackage{amsmath,amsthm,amssymb}

\usepackage{graphicx}
\usepackage{ltxtable} 
\newcolumntype{L}{>{\centering\arraybackslash} m{0.04\columnwidth}} 
\newcolumntype{R}{>{\centering\arraybackslash} m{0.48\columnwidth}} 
\newcolumntype{S}{>{\centering\arraybackslash} m{0.32\columnwidth}} 

\newcommand{\sign}{{\mathrm {sign}}}

\newtheorem{lemma}{Lemma}

\newtheorem{theorem}{Theorem}
\newtheorem{corollary}{Corollary}
\newtheorem{definition}{Definition}

\newtheorem{remark}{Remark}

\DeclareMathOperator*{\E}{\mathbb{E}}

\DeclareMathOperator*{\argmin}{argmin} 
\DeclareMathOperator*{\argmax}{argmax} 
\newcommand{\reals}{\mathbb{R}}

\newcommand{\thmref}[1]{Theorem~\ref{#1}}

\newcommand{\lemref}[1]{Lemma~\ref{#1}}

\newenvironment{myalgo}[1]%
{
\begin{center}
\begin{boxedminipage}{0.8\linewidth}
\begin{center}
\textbf{\texttt{#1}}
\end{center}
\rm
\begin{tabbing}
....\=...\=...\=...\=...\=  \+ \kill
} %
{\end{tabbing} 
\end{boxedminipage} \end{center} 
}

\title{Proximal Stochastic Dual Coordinate Ascent}
\author{Shai Shalev-Shwartz\\
School of Computer Science and Engineering \\
Hebrew University, Jerusalem, Israel
 \and 
Tong Zhang \\
Department of Statistics \\
Rutgers University, NJ, USA}
\date{}

\begin{document}

\maketitle

\begin{abstract}
  We introduce a proximal version of dual coordinate ascent method. We
  demonstrate how the derived algorithmic framework can be used for
  numerous regularized loss minimization problems, including $\ell_1$
  regularization and structured output SVM. The convergence rates we
  obtain match, and sometimes improve, state-of-the-art results.
\end{abstract}

\section{Introduction}

We consider the following generic optimization problem associated with
regularized loss minimization of linear predictors: Let
$X_1,\ldots,X_n$ be matrices in $\reals^{d \times k}$, let
$\phi_1,\ldots,\phi_n$ be a sequence of vector convex functions
defined on $\reals^k$, and $g(\cdot)$ is a convex function defined on
$\reals^d$. Our goal is to solve $\min_{w \in \reals^d} P(w)$ where
\begin{equation} \label{eqn:PrimalProblem}
P(w) = \left[ \frac{1}{n} \sum_{i=1}^n \phi_i( X_i^\top w) + \lambda g(w) \right] ,
\end{equation}
and $\lambda \geq 0$ is a regularization parameter. We will later show
how to use a solver for \eqref{eqn:PrimalProblem} for several popular
regularized loss minimization problems including $\ell_1$
regularization and structured output SVM. 

Let $w^*$ be the optimum of \eqref{eqn:PrimalProblem}. We say that a
solution $w$ is $\epsilon_P$-sub-optimal if $P(w)-P(w^*) \le \epsilon_P$. We
analyze the runtime of optimization procedures as a function of the
time required to find an $\epsilon_P$-sub-optimal solution. 

The dual coordinate ascent (DCA) method solves
a \emph{dual} problem of \eqref{eqn:PrimalProblem}. Specifically, for
each $i$ let $\phi_i^* : \reals^k \to \reals$ be the convex conjugate of $\phi_i$, namely,
$\phi_i^*(u) = \max_{z \in \reals^k} (z^\top u -
\phi_i(z))$. Similarly we define the convex conjugate $g^*$ of $g$. The dual problem is 
\begin{equation} \label{eqn:DualProblem}
\max_{\alpha \in \reals^{k \times n}} D(\alpha) ~~~\textrm{where}~~~ D(\alpha) = 
\left[ \frac{1}{n} \sum_{i=1}^n -\phi_i^*(-\alpha_i) -
  \lambda g^*\left( \tfrac{1}{\lambda n} \sum_{i=1}^n X_i \alpha_i \right) \right] ~,
\end{equation}
where $\alpha_i$ is the $i$'th column of the matrix $\alpha$, which
forms a vector in $\reals^k$.  The dual objective in
(\ref{eqn:DualProblem}) has a different dual vector associated with
each example in the training set.  At each iteration of DCA, the dual
objective is optimized with respect to a single dual vector, while
the rest of the dual vectors are kept intact.
 
We assume that $g^*(\cdot)$ is continuous differentiable. If we define 
\begin{equation} \label{eqn:walpha}
w(\alpha) = \nabla g^*(v(\alpha)) \qquad v(\alpha)= \frac{1}{\lambda
  n} \sum_{i=1}^n X_i \alpha_i ,
\end{equation}
then it is known that $w(\alpha^*)=w^*$, where $\alpha^*$ is an optimal solution of (\ref{eqn:DualProblem}). 
It is also known that $P(w^*)=D(\alpha^*)$ which immediately implies that for all $w$ and $\alpha$, we have
$P(w) \geq D(\alpha)$, and hence the duality gap defined as
\[
P(w(\alpha))-D(\alpha)
\]
can be
regarded as an upper bound on the primal sub-optimality
$P(w(\alpha))-P(w^*)$.

We focus on a \emph{stochastic} version of DCA, abbreviated by SDCA,
in which at each round we choose which dual vector to optimize
uniformly at random.  We analyze SDCA either for $L$-Lipschitz loss
functions or for $(1/\gamma)$-smooth loss functions, which are defined
as follows.
\begin{definition}
 A function $\phi_i: \reals^k \to \reals$ is $L$-Lipschitz if for all $a, b \in \reals^k$, we have
  \[
  |\phi_i(a)- \phi_i(b)| \leq L\,\|a-b\|_P ,
  \]
  where $\|\cdot\|_P$ is a norm.

  A function $\phi_i: \reals^k \to \reals$ is $(1/\gamma)$-smooth if it
  is differentiable and its gradient is $(1/\gamma)$-Lipschitz. An
  equivalent condition is that for all $a, b \in \reals$, we have
  \[
  \phi_i(a) \leq \phi_i(b) + \nabla \phi_i(b)^\top (a-b) + \frac{1}{2\gamma} \|a-b\|_{P}^2 .
  \]
\end{definition}
It is well-known that if $\phi_i(a)$ is $(1/\gamma)$-smooth, then
$\phi_i^*(u)$ is $\gamma$ strongly convex w.r.t.the dual norm:
for all $u, v \in \reals$ and $s \in [0,1]$:
\[
- \phi_i^*( s u + (1-s) v)\geq - s \phi_i^*(u) - (1-s) \phi_i^*(v) + \frac{\gamma s (1-s)}{2} \|u-v\|_D^2 ,
\]
where $\|\cdot\|_D$ is the dual norm of $\|\cdot\|_P$ defined as
\[
\|u\|_D = \sup_{\|v\|_P=1} u^\top v .
\]

We also assume that $g(w)$ is $1$-strongly convex with respect to another
norm $\|\cdot\|_{P'}$:
\[
g(w+ \Delta w) \geq g(w) + \nabla g(w)^\top \Delta w +
\frac{1}{2} \|\Delta w\|_{P'}^2 ,
\]
which means that $g^*(w)$ is $1$-smooth with respect to its dual norm
$\|\cdot\|_{D'}$. Namely,
\begin{equation} \label{eqn:hbound}
g^*(v+ \Delta v) \leq h(v;\Delta v) ~,
\end{equation}
where
\begin{equation} \label{eqn:hdef}
h(v;\Delta v)
 := g^*(v) + \nabla g^*(v)^\top \Delta v +
\frac{1}{2} \|\Delta v\|_{D'}^2 .
\end{equation}

\section{Main Results}

The generic Prox-SDCA algorithm which we analyze in this paper is presented in Figure~\ref{fig:sdca}.
The ideas are described as follows. 
Consider the maximal increase of the dual objective, where we only
allow to change the $i$'th column of $\alpha$. At step $t$, let
$v^{(t-1)} = (\lambda n)^{-1} \sum_i X_i \alpha_i^{(t-1)}$ and let
$w^{(t-1)} = \nabla g^*(v^{(t-1)})$.  We will update the $i$-th dual variable $\alpha_i^{(t)} =
\alpha_i^{(t-1)} + \Delta \alpha_i$, in a way that will lead to a
sufficient increase of the dual objective. 
For primal variable, this would lead to the
update $v^{(t)} =  v^{(t-1)} + (\lambda n)^{-1} X_i \Delta \alpha_i$,
and therefore $w^{(t)} = \nabla g^*(v^{(t)})$ can also be written as
\[
w^{(t)}= \argmax_{w} \left[w^\top v^{(t)}  - g(w) \right] ~=~ 
\argmin_w \left[ - w^\top \left(n^{-1}\sum_{i=1}^n X_i
    \alpha_i^{(t)}\right) + \lambda g(w)\right] ~.
\]
Note that this particular update is rather similar to the update step of
proximal-gradient dual-averaging method in the SGD domain \citep{Xiao10}.
The difference is on how $\alpha^{(t)}$ is updated, and as we will show later, stronger results
can be proved for the Prox-SDCA method when 
we run SDCA for $t >n$ iterations with smooth loss functions.

In order to motivate the proximal SDCA algorithm, we note that the goal of SDCA is to increase the dual objective as much as possible, and thus
the optimal way to choose $\Delta \alpha_i$ would be to maximize the dual objective, namely, we shall let
\[
\Delta \alpha_i = \argmax_{\Delta \alpha_i \in \reals^k} \left[ -\frac{1}{n} \phi^*_i(-(\alpha_i + \Delta
\alpha_i)) - \lambda g^*( v^{(t-1)} + (\lambda n)^{-1}  X_i \Delta
\alpha_i)  \right] ~.
\]
However, for complex $g^*(\cdot)$, this optimization problem may not be easy to solve. 
We will simplify this optimization problem by relying on \eqref{eqn:hbound}. That is, instead of directly maximizing the dual objective function,
we try to maximize the following proximal objective which is a lower bound of the dual objective:
\begin{align*}
&~\argmax_{\Delta \alpha_i \in \reals^k} 
\left[ - \frac{1}{n} \phi^*_i(-(\alpha_i + \Delta
\alpha_i)) - \lambda \left(\nabla g^*(v^{(t-1)})^\top (\lambda n)^{-1}  X_i \Delta
\alpha_i +
\frac{1}{2} \| (\lambda n)^{-1}  X_i \Delta
\alpha_i\|_{D'}^2 \right) \right] \\
=&~\argmax_{\Delta \alpha_i \in \reals^k} 
\left[ -\phi^*_i(-(\alpha_i + \Delta
\alpha_i)) - w^{(t-1)\,\top} X_i \Delta
\alpha_i -
\frac{1}{2\lambda n} \| X_i \Delta
\alpha_i\|_{D'}^2
\right] .
\end{align*}
However, in general, this optimization problem is not necessarily simple to solve. 
We will thus also propose alternative update rules for
$\Delta \alpha_i$ of the form $\Delta \alpha_i = s (u- \alpha_i^{(t-1)})$ for 
an appropriately chosen step size parameter $s>0$ and any vector $u \in \reals^k$ such that
$-u \in \partial \phi_i(X_i^\top w^{(t-1)})$.
Our analysis shows that an appropriate choice of $s$ still leads to a sufficient increase in the dual objective. 

\begin{figure}[htbp]
\begin{myalgo}{Procedure Prox-SDCA} 
\textbf{Parameters}  scalars $\lambda,\gamma$ ($\gamma$ can be $0$), $R$,
norms $\|\cdot\|_D,\|\cdot\|_{D'}$ \\ 
\textbf{Let} $\alpha^{(0)}=0,w^{(0)}=\nabla g^*(0)$ \\
\textbf{Iterate:} for $t=1,2,\dots,T$: \+ \\
 Randomly pick $i$ \\
 Find $\Delta \alpha_i$ using any of the following options (or achieving larger dual objective than one of the options): \+ \\
 \textbf{Option I:} \+ \\ 
  $\Delta \alpha_i \in \argmax_{\Delta \alpha_i}
\left[-\phi_i^*(-(\alpha_i^{(t-1)} + \Delta \alpha_i) ) - 
w^{(t-1)^\top} X_i \Delta \alpha_i - \frac{1}{2\lambda n} \left\| X_i
  \Delta \alpha_i \right\|_{D'}^2\right]$ \- \\
 \textbf{Option II:} \+ \\ 
  Let $u$ be s.t. $-u \in \partial \phi_i(X_i^\top w^{(t-1)})$ \\
   Let $z = u- \alpha_i^{(t-1)} $ \\
   Let $s = \argmax_{s \in [0,1]} \left[-\phi_i^*(-(\alpha_i^{(t-1)} + sz) ) - 
s\,w^{(t-1)^\top} X_i z - \frac{s^2}{2\lambda n} \left\| X_i
  z \right\|_{D'}^2\right]$ \\ 
Set $\Delta \alpha_i = s z$ \- \\
 \textbf{Option III:} \+ \\
  Same as Option II but replace the definition of $s$ as follows: \+ \\
 Let $s = \frac{\phi_i(X_i^\top w^{(t-1)})+\phi_i^*(-\alpha_i^{(t-1)})+ w^{(t-1)^\top} X_i \alpha^{(t-1)}_i
  + \frac{\gamma}{2} \|z\|_D^2}{ \|z\|_D^2 (\gamma +
  \|X_i\|^2 / (\lambda n))}$ \- \- \\
 \textbf{Option IV:} \+ \\
  Same as Option III but replace $\|X_i\|^2$ in the definition of $s$
  with $R^2$ \\
  May also replace $\|z\|_D^2$ with an upper bound no larger than
  $4L^2$ for $L$-Lipschitz non-smooth loss
  \-  \\
 \textbf{Option V (only for smooth losses):} \+ \\
  Set $\Delta \alpha_i = \frac{\lambda n \gamma}{R^2 + \lambda n
    \gamma} ~ \left(- \nabla \phi_i(X_i^\top  w^{(t-1)}) -
    \alpha_i^{(t-1)}\right)$ \-\- \\
 $\alpha^{(t)} \leftarrow \alpha^{(t-1)} + \Delta \alpha_i e_i$ \\
$v^{(t)} \leftarrow v^{(t-1)} + (\lambda n)^{-1} X_i \Delta \alpha_i$ \\
$w^{(t)} \leftarrow \nabla g^*(v^{(t)})$
\- \\
\textbf{Output (Averaging option):} \+ \\
Let $\bar{\alpha}  = \frac{1}{T-T_0} \sum_{i=T_0+1}^T \alpha^{(t-1)}$ \\
Let $\bar{w}  = w(\bar{\alpha}) = \frac{1}{T-T_0} \sum_{i=T_0+1}^T w^{(t-1)}$ \\
return $\bar{w}$  \- \\
\textbf{Output (Random option):} \+ \\
Let $\bar{\alpha}=\alpha^{(t)}$ and $\bar{w}  = w^{(t)}$ for some random $t \in T_0+1,\ldots,T$ \\
return $\bar{w}$ 
\end{myalgo}
\caption{The Generic Proximal Stochastic Dual Coordinate Ascent Algorithm}
\label{fig:sdca}
\end{figure}

We analyze the algorithm based on different assumptions on the loss
functions. To simplify the statements of our theorems, we always
make the following assumptions:
\begin{itemize}
\item Assume that the loss functions satisfy
\[
\frac{1}{n} \sum_{i=1}^n \phi_i(0) \le 1 \quad \textrm{and} \quad \forall i, a , ~~ \phi_i(a) \ge
0 ~~.
\]
\item Assume that $\max_i \|X_i\| \leq R$, where
\[
\|X_i\| = \sup_{u \neq 0} \frac{\|X_i u\|_{D'}}{\|u\|_D} .
\]
\end{itemize}

Under the above assumptions, we have the following convergence result for smooth loss functions.
\begin{theorem} \label{thm:smooth} 
  Consider Procedure Prox-SDCA.  Assume that $\phi_i$ is
  $(1/\gamma)$-smooth for all $i$.  To obtain an expected duality gap
  of $\E [P(w^{(T)})-D(\alpha^{(T)})] \leq \epsilon_P$, it suffices to have a total number of
  iterations of
\[
T \geq \left(n +
  \tfrac{R^2}{\lambda \gamma}\right) \, \log( (n + \tfrac{R^2}{\lambda \gamma})   \cdot \tfrac{1}{\epsilon_P}) .
\]
Moreover, to obtain an expected duality gap of $\E [P(\bar{w})-D(\bar{\alpha})] \leq \epsilon_P$, it suffices to have a total number 
of iterations of
\[
T_0 \geq \left(n +
  \frac{R^2}{\lambda \gamma}\right) \, \log( (n + \tfrac{R^2}{\lambda \gamma})   \cdot \tfrac{1}{(T-T_0)\epsilon_P}) .
\]
\end{theorem}

The linear convergence result in the above theorem is faster than the corresponding proximal SGD result when $T \gg n$. 
This indicates the advantage of Proximal SDCA approach when we run more than one pass over the data. 
Similar results can also be found in  \cite{CollinsGlKoCaBa08,LSB12-sgdexp,ShZh12-sdca} but in more restricted settings
than the general problem considered in this paper.
Unlike traditional batch algorithms (such as proximal gradient descent, or accelerated proximal gradient descent)
that can only achieve relatively fast convergence when the condition number $1/(\lambda \gamma))=O(1)$,
our algorithm allows relatively fast convergence even when the condition number $1/(\lambda \gamma))=O(n)$, which can be a significant
improvement for real applications.

For nonsmooth loss functions, the convergence rate for Prox-SDCA is given below. 
\begin{theorem} \label{thm:Lipschitz}
Consider Procedure Prox-SDCA. 
Assume that $\phi_i$ is $L$-Lipschitz for all $i$.
To obtain an expected duality gap of $\E [P(\bar{w})-D(\bar{\alpha})] \leq \epsilon_P$, it suffices to have a total number of
iterations of
\[
T \geq T_0 + n + \frac{4 \,(RL)^2}{\lambda \epsilon_P} \geq 
\max(0, \lceil n \log(0.5 \lambda n (RL)^{-2} ) \rceil ) + n + \frac{20 \,(RL)^2}{\lambda \epsilon_P} ~.
\]
Moreover, when $t \geq T_0$, we have dual sub-optimality bound of
$\E [D(\alpha^*) - D(\alpha^{(t)})] \leq \epsilon_P/2$.
\end{theorem}

The result shown in the above theorem for nonsmooth loss is comparable to that of proximal SGD.  
However, one advantage of our result is that
the convergence is in duality gap, which can be easily checked during the algorithm to serve as a stopping criterion. 
In comparison, SGD does not have an easy to implement stopping criterion.
Moreover, as discussed in \cite{ShZh12-sdca}, faster convergence (such as linear convergence) can be obtained asymptotically
when the nonsmooth loss function is nearly everywhere smooth, and in such case, the practical performance of the algorithm will be
superior to SGD when we run more than one pass over the data.

\section{Applications}

There are numerous possible applications of our algorithmic
framework. Here we list three applications. 

\subsection{$\ell_1$ regularization assuming instances of low $\ell_2$ norm}

Suppose our interest is to solve $\ell_1$ regularization problem of
the form
\begin{equation} \label{eqn:l1regu}
\min_w \left[ \frac{1}{n} \sum_{i=1}^n \phi_i( x_i^\top w) + \sigma
  \|w\|_1 \right] ~,
\end{equation}
with a positive regularization parameter $\sigma \in \reals_+$.
Assume also that $R = \max_i \|x_i\|_2$
is not too large. This would be the case, for example, in text
categorization problems where each $x_i$ is a bag-of-words representation
of some short document. 

Let $w^*$ be an optimal solution of \eqref{eqn:l1regu} and
assume\footnote{We can always take $B = 1/\sigma$ since by the
  optimality of $w^*$ we have $\|w^*\|_2 \le \|w^*\|_1 \le 1/\sigma$.}
that $\|w^*\|_2 \le B$.  Choose $\lambda = \frac{\epsilon}{B^2}$ and
\begin{equation} \label{eqn:gdefl1l2}
g(w) = \frac{1}{2} \|w\|_2^2 + \frac{\sigma}{\lambda} \|w\|_1 ~.
\end{equation}
Consider the problem:
\begin{equation} \label{eqn:l1l2regu}
\min_w P(w) := \left[ \frac{1}{n} \sum_{i=1}^n \phi_i( x_i^\top w) +\lambda g(w) \right] ~.
\end{equation}
Then, if $\hat{w}$ is an
$(\epsilon/2)$-approximated solution of the above it holds that
\[
\frac{1}{n} \sum_{i=1}^n \phi_i( x_i^\top \hat{w}) + \sigma
\|\hat{w}\|_1 \le P(\hat{w}) \le P(w^*) +
\frac{\epsilon}{2} \le \frac{1}{n} \sum_{i=1}^n \phi_i( x_i^\top w^*)
+ \sigma \|w^*\|_1 + \epsilon ~.
\]
It follows that $\hat{w}$ is an $\epsilon$-approximated solution to
the problem \eqref{eqn:l1regu}. Hence, we can focus on solving
\eqref{eqn:l1l2regu} based on the Prox-SDCA framework.
Note that if our goal is to solve a general $L_1$-$L_2$ regularization problem with a fixed $\lambda$ independent of $\epsilon$,
then linear convergence can be obtained from our analysis when the loss functions are smooth. However, this section
focuses on the case that our interest is to solve \eqref{eqn:l1regu}, and thus $\lambda$ is chosen according to $\epsilon$.
The reason to introduce an extra $\ell_2$ regularization in \eqref{eqn:gdefl1l2} is because our theory requires $g(w)$ to be
$1$-strongly convex, which is satisfied by \eqref{eqn:gdefl1l2} with respect to the $\ell_2$-norm.

To derive the actual algorithm, we first need to calculate the
gradient of the conjugate of $g$. We have
\begin{align*}
\nabla g^*(v) &= \argmax_{w} \left[w^\top v - \frac{1}{2} \|w\|_2^2 -
\frac{\sigma}{\lambda} \|w\|_1 \right] \\
&= \argmin_w \left[ \frac{1}{2} \|w-v\|_2^2 +
  \frac{\sigma}{\lambda} \|w\|_1 \right]
\end{align*}
A sub-gradient of the objective of the optimization problem above is
of the form $w-v +
\frac{\sigma}{\lambda} z = 0$, where $z$ is a vector with $z_i =
\sign(w_i)$, where if $w_i=0$ then $z_i \in [-1,1]$. Therefore, if $w$
is an optimal solution then for all $i$, either $w_i=0$ or $w_i = v_i
- \frac{\sigma}{\lambda} \sign(w_i)$. Furthermore, it is easy to
verify that if $w$ is an optimal solution then for all $i$, if $w_i
\neq 0$ then the sign
of $w_i$ must be the sign of $v_i$. Therefore, whenever $w_i \neq 0$
we have that  $w_i = v_i - \frac{\sigma}{\lambda} \sign(v_i)$. It
follows that in that case we must have $|v_i| >
\frac{\sigma}{\lambda}$. And, the other direction is also true,
namely, if $|v_i| > \frac{\sigma}{\lambda}$ then setting $w_i = v_i -
\frac{\sigma}{\lambda} \sign(v_i)$ leads to an objective value of 
\[
\left(\frac{\sigma}{\lambda}\right)^2 + \frac{\sigma}{\lambda} (|v_i| -
\frac{\sigma}{\lambda}) \le |v_i|^2 ~,
\]
where the right-hand side is the objective value we will obtain by
setting $w_i=0$. This leads to the conclusion that 
\[
\nabla_i g^*(v) = \sign(v_i)\left[ |v_i| - \tfrac{\sigma}{\lambda}\right]_+ = \begin{cases}
v_i - \frac{\sigma}{\lambda} \sign(v_i) & \textrm{if}~ |v_i| >
\frac{\sigma}{\lambda} \\
0 & \textrm{o.w.}
\end{cases}
\]

The resulting algorithm is as follows:

\begin{myalgo}{Procedure Prox-SDCA for minimizing
    \eqref{eqn:l1regu}  using $g$ as in \eqref{eqn:gdefl1l2}} 
\textbf{Parameters} \+ \\
 regularization $\sigma$  \\
 target accuracy $\epsilon$ \\
 $B \ge \|w^*\|_2$ (default value $B=1/\sigma$) \- \\
Run Prox-SDCA with: \+ \\
$\|\cdot\|_D = |\cdot|$, $\|\cdot\|_{D'} = \|\cdot\|_2$, and $R \geq
\max_i \|x_i\|_2$ \\
$\lambda = \epsilon/B^2$ \\
$\nabla_i g^*(v) = \sign(v_i)\left[ |v_i| - \tfrac{\sigma}{\lambda}\right]_+ $
\end{myalgo}

In terms of runtime, we obtain the following result from the general theory, where
the notation $\tilde{O}(\cdot)$ ignores any log-factor. 
\begin{corollary} 
The number of iterations required by Prox-SDCA, with $g(\cdot)$ as in
\eqref{eqn:gdefl1l2},  for solving
\eqref{eqn:l1regu} to an accuracy $\epsilon$ is 
\begin{align*}
\tilde{O}\left(n + \frac{R^2 B^2}{\epsilon\,\gamma}\right) &
~~~\mathrm{if}~\forall
i,~\phi_i~\mathrm{is}~(1/\gamma)~\mathrm{-smooth} \\
\tilde{O}\left(n + \frac{L^2 R^2 B^2}{\epsilon^2}\right) & 
~~~\mathrm{if}~\forall
i,~\phi_i~\mathrm{is}~(L)~\mathrm{-Lipschitz} 
\end{align*}
In both cases, $R$ is an upper bound of $\max_i \|x_i\|_2$ and $B$ is an upper bound on
$\|w^*\|_2$.   
\end{corollary}

\subsubsection*{Related Work}

Standard SGD requires $O(R^2 B^2 / \epsilon^2)$ even in the case of
smooth loss functions. Several variants of SGD, that leads to sparser
intermediate solutions, have been proposed
(e.g. \cite{LangfordLiZh09,shalev2011stochastic,Xiao10,duchi2009efficient,DuchiShSiTe10}). However,
all of these variants share the iteration bound of $O(R^2 B^2 /
\epsilon^2)$, which is slower than our bound when $\epsilon$ is small.

Another relevant approach is the FISTA algorithm of 
\cite{beck2009fast}. The shrinkage operator of FISTA is the same as
the gradient of $g^*$ used in our approach.  
It is a batch algorithm using Nesterov's accelerated gradient technique.
For smooth loss
functions, FISTA enjoys the iteration bound of 
\[
O\left(  \frac{RB}{\sqrt{\epsilon ~ \gamma}} \right)  ~.
\]
However, each iteration of FISTA involves all the $n$ examples rather
than just a single example, as our method. Therefore, the runtime of
FISTA would be
\[
O\left( d\,n\, \frac{R B}{\sqrt{\epsilon ~ \gamma}} \right)  ~.
\]
In contrast, the runtime of Prox-SDCA is
\[
\tilde{O}\left(d\left(n + \frac{R^2 B^2}{\epsilon\,\gamma}\right)\right) ~,
\]
which is better when $n \gg \frac{RB}{\sqrt{\epsilon\,\gamma}}$. This happens
in the statistically interesting regime where we usually choose $\epsilon$ larger than $\Omega(1/n^2)$
for machine learning problems.
In fact, since the generalization performance of a learning algorithm is in general no better than $O(1/n)$, 
there is no need to choose $\epsilon = o(1/n)$. This means that in the statistically interesting regime,
Prox-SDCA is superior to FISTA.

Another approach to solving \eqref{eqn:l1regu} when the loss functions
are smooth is stochastic coordinate descent over the primal
problem. \cite{shalev2011stochastic} showed that the runtime of this
approach is 
\[
O\left(\frac{dnB^2}{\epsilon}\right) ~,
\]
under the assumption that $\|x_i\|_\infty \le 1$ for all $i$. Similar results can also be found in \cite{Nesterov10}.

For our method, each iteration costs runtime $O(d)$ so the total
runtime is
\[
\tilde{O}\left(d\left(n + \frac{R^2B^2}{\epsilon}\right)\right) ~,
\]
where $R = \max_i \|x_i\|_2$.  Since the assumption $\|x_i\|_\infty
\le 1$ implies $R^2 \le d$, this is similar to the guarantee of
\cite{shalev2011stochastic} in the worst-case. However, in many
problems, $R^2$ can be a constant that does not depend on $d$
(e.g. when the instances are sparse). In that case, the runtime of
Prox-SDCA becomes $\tilde{O}\left(d(n+B^2/\epsilon)\right)$,
which is much better than the runtime bound for the primal stochastic coordinate descent
method given in \cite{shalev2011stochastic}.

\subsection{$\ell_1$ regularization with low $\ell_\infty$ instances}

Next, we consider \eqref{eqn:l1regu} but now we assume that $R =
\max_i \|x_i\|_\infty$ is not too large (but $\max_i \|x_i\|_2$ might
be large).
This is the situation considered in \cite{shalev2011stochastic}.

Let $w^*$ be an optimal solution of \eqref{eqn:l1regu} and
assume\footnote{We can always take $B = 1/\sigma$ since by the
  optimality of $w^*$ we have $\|w^*\|_1 \le 1/\sigma$.}
that $\|w^*\|_1 \le B$.  Choose $\lambda = \frac{\epsilon}{3\log(d)B^2}$ and
\begin{equation} \label{eqn:gdefqnorm}
g(w) = \frac{3\log(d)}{2} \|w\|_q^2 + \frac{\sigma}{\lambda} \|w\|_1 ~,
\end{equation}
where $q = \frac{\log(d)}{\log(d)-1}$. The function $g(w)$ is $1$-strongly
convex with respect to the norm $\|\cdot\|_1$ over $\reals^d$ (see for example
\cite{KakadeShTe12}). Consider the problem \eqref{eqn:l1l2regu} with
$g(\cdot)$ being defined in \eqref{eqn:gdefqnorm}.
As before, if $\hat{w}$ is an $(\epsilon/2)$-approximated solution of
the above problem then it is also an $\epsilon$-approximated solution
to the problem \eqref{eqn:l1regu}. Hence, we can focus on solving
\eqref{eqn:l1l2regu} based on the Prox-SDCA framework.

To derive the actual algorithm, we need to calculate the
gradient of the conjugate of $g$. We have
\begin{align*}
\nabla g^*(v) &= \argmin_{w} \left[-w^\top v + \frac{3\log(d)}{2} \|w\|_q^2 +
\frac{\sigma}{\lambda} \|w\|_1 \right] .
\end{align*}
The $i$'th component of a sub-gradient of the objective of the
optimization problem above is of the form 
\[ - v_i + \frac{3\log(d) \sign(w_i)|w_i|^{q-1}}{\|w\|_q^{q-2}} +
\frac{\sigma}{\lambda} z_i  ~,
\] where $z_i =
\sign(w_i)$ whenever $w_i \neq 0$ and otherwise $z_i \in [-1,1]$. Therefore, if $w$
is an optimal solution then for all $i$, either $w_i=0$ or 
\[
|w_i|^{q-1} = \sign(w_i) ~\frac{\|w\|_q^{q-2}}{3\log(d)}~ \left(v_i
- \frac{\sigma}{\lambda} \sign(w_i)\right) =
\frac{\|w\|_q^{q-2}}{3\log(d)}~ \left(\sign(w_i) ~ v_i
- \frac{\sigma}{\lambda}\right) .
\]
 Furthermore, it is easy to
verify that if $w$ is an optimal solution then for all $i$, if $w_i
\neq 0$ then the sign
of $w_i$ must be the sign of $v_i$. Therefore, whenever $w_i \neq 0$
we have that  
\[
|w_i|^{q-1} = \frac{\|w\|_q^{q-2}}{3\log(d)}~  \left(|v_i|
- \frac{\sigma}{\lambda}\right) .
\]
It follows that in that case we must have $|v_i| >
\frac{\sigma}{\lambda}$. And, the other direction is also true,
namely, if $|v_i| > \frac{\sigma}{\lambda}$ then $w_i$ must be
non-zero. This is true because if $|v_i| >
\frac{\sigma}{\lambda}$, then the $i$'th coordinate of any
sub-gradient of the objective function at any vector $w$ s.t. $w_i=0$
is $-v_i + \frac{\sigma}{\lambda} z_i \neq 0$. Hence, $w$ can't be an
optimal solution. This leads to the conclusion that an optimal
solution has the form
\begin{equation} \label{eqn:nablaDefqnorm}
\nabla_i g^*(v) = \begin{cases}
\sign(v_i)~\left(a~\left(|v_i|
- \frac{\sigma}{\lambda}\right)\right)^{\frac{1}{q-1}} & \textrm{if}~|v_i|
> \frac{\sigma}{\lambda} \\
0 & \textrm{otherwise}
\end{cases}
~,
\end{equation}
where
\begin{align*}
a &= \frac{\|\nabla g^*(v)\|_q^{q-2}}{3\log(d)} = \frac{1}{3\log(d)} \left(\sum_{i : |v_i|
> \frac{\sigma}{\lambda} } \left(a \left(|v_i|
-
\frac{\sigma}{\lambda}\right)\right)^{\frac{q}{q-1}}\right)^{\frac{q-2}{q}}
= \frac{a^{\frac{q-2}{q-1}}}{3\log(d)} \left(\sum_{i : |v_i|
> \frac{\sigma}{\lambda} } \left(|v_i|
-
\frac{\sigma}{\lambda}\right)^{\frac{q}{q-1}}\right)^{\frac{q-2}{q}}
~, 
\end{align*}
which yields
\begin{equation} \label{eqn:aDefqnorm}
a =  \left(\frac{1}{3\log(d)} \left(\sum_{i : |v_i|
> \frac{\sigma}{\lambda} } \left(|v_i|
-
\frac{\sigma}{\lambda}\right)^{\frac{q}{q-1}}\right)^{\frac{q-2}{q}} \right)^{q-1}~.
\end{equation}

The resulting algorithm is as follows:

\begin{myalgo}{Procedure Prox-SDCA for minimizing
    \eqref{eqn:l1regu} using $g$ as in \eqref{eqn:gdefqnorm}} 
\textbf{Parameters} \+ \\
 regularization $\sigma$  \\
 target accuracy $\epsilon$ \\
 dimension $d$ \\
 $B \ge \|w^*\|_1$ (default value $B=1/\sigma$) \- \\
Run Prox-SDCA with: \+ \\
$\|\cdot\|_D = |\cdot|$, $\|\cdot\|_{D'} = \|\cdot\|_\infty$, and $R \geq
\max_i \|x_i\|_\infty$ \\
$\lambda = \frac{\epsilon}{3\log(d) B^2}$ \\
$\nabla g^*(v)$ according to \eqref{eqn:nablaDefqnorm} and \eqref{eqn:aDefqnorm}
\end{myalgo}

In terms of runtime, we obtain the following
\begin{corollary} 
The number of iterations required by Prox-SDCA, with $g$ as in
\eqref{eqn:gdefqnorm}, for solving \eqref{eqn:l1regu} to accuracy $\epsilon$ is 
\begin{align*}
\tilde{O}\left(n + \frac{R^2 B^2 \log(d)}{\epsilon\,\gamma}\right) &
~~~\mathrm{if}~\forall
i,~\phi_i~\mathrm{is}~(1/\gamma)~\mathrm{-smooth} \\
\tilde{O}\left(n + \frac{L^2 R^2 B^2 \log(d)}{\epsilon^2}\right) & 
~~~\mathrm{if}~\forall
i,~\phi_i~\mathrm{is}~(L)~\mathrm{-Lipschitz} 
\end{align*}
In both cases, $R = \max_i \|x_i\|_\infty$ and $B$ is an upper bound over
$\|w^*\|_1$.   
\end{corollary}

\subsubsection*{Related work}

The algorithm we have obtained is similar to the Mirror Descent
framework \cite{BeckTe03} and its online or stochastic versions (see
for example \cite{OLsurvey} and the references therein). It is also
closely related to the SMIDAS and COMID algorithms
\cite{shalev2011stochastic} as well as to dual averaging
\cite{Xiao10}. Comparing the rates of these algorithms to Prox-SDCA,
we obtain similar differences as in the previous subsection, only now
$B$ is a bound on $\|w^*\|_1$ rather than $\|w^*\|_2$ and $R$ is a
bound on $\max_i \|x_i\|_\infty$ rather than $\max_i \|x_i\|_2$.

\subsection{Multiclass categorization and structured prediction}

In structured output problems, there is an instance space
$\mathcal{X}$ and a large target space $\mathcal{Y}$. There is a
function $\psi : \mathcal{X} \times \mathcal{Y} \to \reals^d$. We
assume that the range of $\psi$ is in the $\ell_2$ ball of radius $R$
of $\reals^d$. The prediction of a vector $w \in \reals^d$ is
\[
\argmax_{y \in \mathcal{Y}} w^\top \psi(x,y) ~.
\]
There is also a function $\delta : \mathcal{Y} \times \mathcal{Y} \to
\reals_+$ which evaluates the cost of predicting a label $y'$ when the
true label is $y$. We assume that $\delta(y,y) = 0$ for all $y$. The
generalized hinge-loss defined below is used as a convex surrogate loss function
\[
\max_{y'} \left[ \delta(y',y) - w^\top \psi(x,y) + w^\top \psi(x,y') \right] ~.
\]
The optimization problem associated with learning $w$ is now
\begin{equation} \label{eqn:SO1}
\min_w  ~\left[ \frac{\lambda}{2} \|w\|_2^2 + \frac{1}{n} \sum_{i=1}^n \left(
\max_{y'} \delta(y',y_i) - w^\top \psi(x_i,y_i) + w^\top \psi(x_i,y') \right)
\right] ~.
\end{equation}

The above optimization problem can be cast in our setting as
follows. W.l.o.g. assume that $\mathcal{Y} = \{1,\ldots,k\}$. 
For each $i$ and each $j$, let the $j$'th column
of $X_i$ be $\psi(x_i,j)$. Define,
\[
\phi_i(v) = \max_j \left(\delta(j,y_i) - v_{y_i} + v_j \right) ~.
\]
Finally, let $g(w) = \frac{1}{2} \|w\|_2^2$.
Then, \eqref{eqn:SO1} can be written in the form of \eqref{eqn:PrimalProblem}.

To apply the Prox-SDCA to this problem, note that $g$ is $1$-strongly
convex w.r.t. $\|\cdot\|_2$ and that $\phi_i$ is $2$-Lipschitz
w.r.t. norm $\|\cdot\|_\infty$. Indeed, given vectors $u,v$, let $j$
be the index that attains the maximum in the definition of
$\phi_i(v)$, then
\[
\phi_i(v)-\phi_i(u) \le \left(\delta(j,y_i) - v_{y_i} + v_j \right) - 
\left(\delta(j,y_i) - u_{y_i} + u_j \right) \le 2 \|v-u\|_\infty ~.
\]
Therefore $\|\cdot\|_D = \|\cdot\|_1$ and $\|\cdot\|_{D'} =
\|\cdot\|_2$. 
If we let
\[
R \geq \max_j \|\psi(x_i,j)\|_2 ,
\]
then we have that
\[
\|X_i\| = \sup_{u \neq 0} \frac{\|X_i u\|_2}{\|u\|_1} = \sup_{u :
  \|u\|_1 = 1} \|X_i u\|_2 = \max_j \|\psi(x_i,j)\|_2 \le  R ~.
\]

To calculate the dual of $\phi_i$, note that we can write $\phi_i$ as 
\[
\phi_i = \max_{\beta \in \Delta^k} \sum_j \beta_j \left(\delta(j,y_i) - v_{y_i} + v_j \right) ~,
\]
where $\Delta^k=\{\beta: \sum_j \beta_j \leq 1; \beta_j \geq 0\}$ is the non-negative simplex of $\reals^k$. Hence, the dual of $\phi_i$ is
\begin{align*}
\phi^*_i(\alpha) &= \max_{v} \left[ v^\top \alpha - \phi_i(v) \right] \\
&= \max_{v} \min_{\beta} \left[ v^\top \alpha - \sum_j \beta_j
\left(\delta(j,y_i) - v_{y_i} + v_j \right) \right] \\
&= \min_{\beta} \max_{v} \left[ v^\top \alpha - \sum_j \beta_j
\left(\delta(j,y_i) - v_{y_i} + v_j \right)  \right] \\
&= \min_{\beta} \left[ \beta^\top \delta(\cdot,y_i) + \max_v \left[v^\top
  (\alpha-\beta) + v_{y_i} \sum_j \|\beta\|_1 \right]\right] .
\end{align*}
The inner maximization over $v$ would be $\infty$ if for some $j \neq
y_i$ we have $\alpha_j \neq \beta_j$. Otherwise, if 
for all $j \neq y_i$ we have $\alpha_j = \beta_j$ the inner objective becomes
\[
v_{y_i} \big(\alpha_{y_i} - \beta_{y_i} + \sum_{j} \beta_j\big) = 
v_{y_i} \big(\alpha_{y_i} + \sum_{j \neq y_i} \alpha_j\big) ~. 
\]
Therefore, the objective would again be $\infty$ if $\alpha_{y_i} \neq
- \sum_{j \neq y_i} \alpha_j$. In all other cases, the objective is
zero. Overall, this implies that:
\[
\phi^*_i(\alpha) = \begin{cases}
\sum_j \alpha_j \delta(j,y_i) & \textrm{if}~ \sum_j \alpha_j = 0~\land~\forall
j \neq y_i, \alpha_j \ge 0 ~\land~ \sum_{j \neq y_i} \alpha_j \le 1 \\
\infty &\textrm{o.w.}
\end{cases}
\]

Finally, we specify Prox-SDCA (using Option IV with $2$ as an upper bound of $\|z\|_D$, 
and the random output
option), and rely on the fact that a sub-gradient of $\phi_i(v)$ is a
vector $e_j - e_{y_i}$ with $j \in \argmax_j \left(\delta(j,y_i) -
  v_{y_i} + v_j \right) $.

\begin{myalgo}{Procedure Prox-SDCA for structured output learning} 
\textbf{Parameter}  scalar $\lambda$ \\ 
\textbf{Let} $w^{(0)}=0$ ~;~ $R \ge \max_{i,j}
\|\phi(x,j)\|_2$ ~ \+ \\
  $\forall_i, w_i^{(0)} = 0$  (we'll maintain
  $w_i^{(t)} = (\lambda n)^{-1} X_i \alpha_i^{(t)}$ and $w^{(t)}=\sum_i w_i^{(t)}$) \\
  $\forall_i, D_i^{(0)} = 0$  (we'll maintain
  $D_i^{(t)}=\phi_i^*(\alpha^{(t)})$) \- \\ 
\textbf{Iterate:} for $t=1,2,\dots,T$: \+ \\
 Randomly pick $i$ \\
 Let $j \in \argmax_{j} \left(\delta(j,y_i) - w^{(t-1)~\top} \phi(x_i,y_i) +
   w^{(t-1)~\top} \phi(x_i,j) \right)$ \\
 Let $P_i = \phi_i(X_i^\top w^{(t-1)}) =  \delta(j,y_i) - w^{(t-1)~\top} \phi(x_i,y_i) +
   w^{(t-1)~\top} \phi(x_i,j) $ \\
 Let $s = \frac{P_i+D_i^{(t-1)}+ \lambda n w^{(t-1)^\top} w_i^{(t-1)}}{ 
        4 R^2 / (\lambda n)}$ \\
$D^{(t)}_i \leftarrow (1-s)D^{(t-1)}_i + s\,\delta(j,y_i)$ \\
$w_i^{(t)} \leftarrow (1-s)w_i^{(t-1)} + s (\lambda n)^{-1}
(\phi(x_i,y_i)-\phi(x_i,j))$ \\
$w^{(t)} \leftarrow w^{(t-1)} + w^{(t)}_i - w^{(t-1)}_i$
\- \\
\textbf{Output:} \+ \\
Return $\bar{w}  = w^{(t)}$ for some random $t \in T_0+1,\ldots,T$ 
\end{myalgo}

Note that even if $k$ is very large, the above implementation does not
maintain $\alpha$ explicitly, but only maintains $d$-dimensional
vectors. Therefore, we can implement the above procedure efficiently
whenever the optimization problem involves in finding $j$ can be
performed efficiently. This is the same requirement as in implementing
SGD for structured output prediction. 

\begin{corollary}
Prox-SDCA can be implemented for structured output prediction. To
obtain an expected duality gap of at most $\epsilon_P$, it suffices to
have a total number of iterations of
\[
T \geq 
\max(0, \lceil n \log(0.5 \lambda n (2R)^{-2} ) \rceil ) + n + \frac{20 \,(2R)^2}{\lambda \epsilon_P} ~,
\]
where $R$ is an upper bound on $\|\phi(x_i,j)\|_2$. The most expensive
operation at iteration $t$ is solving
\begin{equation} \label{eqn:mostExpensive}
\argmax_{j} \left(\delta(j,y_i) - w^{(t-1)~\top} \phi(x_i,y_i) +
   w^{(t-1)~\top} \phi(x_i,j) \right) ~.
\end{equation}
\label{cor:struct-svm}
\end{corollary}

\begin{remark}
  Since for this problem, $\|z\|_D^2$ in Option IV can be bounded by $L^2=4$ instead of $4L^2=16$, the proof of Theorem~\ref{thm:Lipschitz}
  implies that the constant 20 in Corollary~\ref{cor:struct-svm} can be replaced by $5$.
\end{remark}

\subsubsection*{Related Work}

For structured prediction problem, SGD enjoys the rate
\[
\tilde{O}\left( \frac{R^2}{\lambda \epsilon} \right) ~,
\]
while the most expensive operation at each iteration of SGD also
involves solving \eqref{eqn:mostExpensive}. Therefore, our bound
matches the bound of SGD when $n = \tilde{O}\left( \frac{R^2}{\lambda
    \epsilon} \right)$. 
The main advantage of our result is that it bounds duality gap which can be checked in practice. Moreover, the
practical convergence speed can be faster than what is indicated in Corollary~\ref{cor:struct-svm} when the non-smooth loss function
can be approximated by a smooth loss function, as pointed out in \cite{ShZh12-sdca}.

Recently, \cite{lacoste2012stochastic} derived a stochastic coordinate
ascent for structural SVM based on the Frank-Wolfe
algorithm. Their algorithm is very similar to our algorithm and the
rate they obtain for the convergence of duality gap matches our rate. 

Note that the generality of our framework enables us to easily handle
structured output problems with other regularizers, such as $\ell_1$
norm regularization. 

\section{Proofs}

Note that the proof technique follows that of \cite{ShZh12-sdca}, but with more involved notations of the paper.
We prove the theorems for running Prox-SDCA while choosing $\Delta
\alpha_i$ as in Option I. A careful examination of the proof easily
reveals that the results hold for the other options as well. More specifically,
Lemma~\ref{lem:key} only requires choosing $\Delta \alpha_i = s (u_i^{(t-1)}-\alpha_i^{(t-1)})$ as in \eqref{eqn:PC1},
and Option III chooses $s$ to optimize the bound on the right hand side of \eqref{eqn:PC3}, and hence ensures
that the choice can do no worse than the result of Lemma~\ref{lem:key} with any $s$. The simplification in Option IV and V
employs the specific simplification of the bound in Lemma~\ref{lem:key} in the proof of the theorems.

For convenience, we list the following simple facts about primal and
dual formulations, which will be used in the proofs.
For each $i$, we have
\[
-\alpha_i^* \in \partial \phi_i(X_i^\top w^{*}) , \quad
X_i^\top w^{*} \in \partial \phi_i^*(-\alpha_i^*) ,
\]
and 
\[
w^* = \nabla g^*(v^*) , \quad v^*= \frac{1}{\lambda n} \sum_{i=1}^n X_i \alpha_i^* .
\]

The key lemma is the following:
\begin{lemma} \label{lem:key}
Assume that $\phi^*_i$ is $\gamma$-strongly-convex (where $\gamma$ can
be zero). Then, for any iteration $t$ and any $s \in [0,1]$ we have
\[
\E[D(\alpha^{(t)})-D(\alpha^{(t-1)})] \ge  \frac{s}{n}\,
\E \; [P(w^{(t-1)})-D(\alpha^{(t-1)})] - \left(\frac{s}{n}\right)^2
\frac{G^{(t)}}{2\lambda} ~,
\]
where
\[
G^{(t)} = \frac{1}{n} \sum_{i=1}^n \left(\|X_i\|^2 -
      \frac{\gamma(1-s)\lambda n}{s}\right) \; \E \left[\|u^{(t-1)}_i-\alpha^{(t-1)}_i\|_D^2\right] ,
\]
and $-u^{(t-1)}_i \in \partial \phi_i(X_i^\top w^{(t-1)})$.
\end{lemma}
\begin{proof}
Since only the $i$'th element of $\alpha$ is updated, the improvement in the dual objective can be written as
\begin{align*}
& n[D(\alpha^{(t)}) - D(\alpha^{(t-1)})] \\
= &
\left(-\phi^*(-\alpha^{(t)}_i) - \lambda n g^*\left(v^{(t-1)} + (\lambda
  n)^{-1} X_i \Delta \alpha_i\right) \right) -
\left(-\phi^*(-\alpha^{(t-1)}_i) - \lambda n g^*\left(v^{(t-1)}\right) \right) \\
\geq &
\underbrace{\left(-\phi^*(-\alpha^{(t)}_i) - \lambda n h\left(v^{(t-1)}; (\lambda
  n)^{-1} X_i \Delta \alpha_i\right)\right) }_A -
\underbrace{\left(-\phi^*(-\alpha^{(t-1)}_i) - \lambda n g^*\left(v^{(t-1)}\right) \right)}_B .
\end{align*}

By the definition of the update we have for all $s \in [0,1]$ that
\begin{align} \nonumber
A &=  \max_{\Delta \alpha_i} -\phi^*(-(\alpha^{(t-1)}_i + \Delta\alpha_i)) - 
\lambda n h\left(v^{(t-1)}; (\lambda
  n)^{-1} X_i \Delta \alpha_i\right) \\
&\ge -\phi^*(-(\alpha^{(t-1)}_i + s(u^{(t-1)}_i - \alpha^{(t-1)}_i) ))
- \lambda n
h(v^{(t-1)}; (\lambda n)^{-1} s X_i (u^{(t-1)}_i -\alpha^{(t-1)}_i)) .
\label{eqn:PC1}
\end{align}

From now on, we omit the superscripts and subscripts. 
Since $\phi^*$ is $\gamma$-strongly convex, we have that
\begin{equation} \label{eqn:PC2}
\phi^*(-(\alpha+ s(u - \alpha) )) = \phi^*(s (-u) + (1-s) (-\alpha))
\le s \phi^*(-u) + (1-s) \phi^*(-\alpha) - \frac{\gamma}{2} s (1-s) \|u-\alpha\|_D^2
\end{equation}
Combining this with \eqref{eqn:PC1} and rearranging terms we obtain that
\begin{align*} 
A &\ge -s \phi^*(-u) - (1-s) \phi^*(-\alpha) + \frac{\gamma}{2} s (1-s)
\|u-\alpha\|_D^2
- \lambda n
  h(v; (\lambda n)^{-1} s X(u - \alpha) )  \\
&= -s \phi^*(-u) - (1-s) \phi^*(-\alpha) + \frac{\gamma}{2} s (1-s)
\|u-\alpha\|_{D}^2
- \lambda n g^*(v) - s w^\top X (u-\alpha) - 
\frac{s^2}{2\lambda n} \|X(u-\alpha)\|_{D'}^2 \\
&\geq -s \phi^*(-u) - (1-s) \phi^*(-\alpha) + \frac{\gamma}{2} s (1-s)
\|u-\alpha\|_D^2
- \lambda n g^*(v) - s w^\top X (u-\alpha) - 
\frac{s^2}{2\lambda n} \|X\|^2 \|u-\alpha\|_D^2 \\
&= \underbrace{-s(\phi^*(-u)+w^\top X u)}_{s\,\phi(X^\top w)} + \underbrace{(-\phi^*(-\alpha) - \lambda
  n g^*(v))}_B + \frac{s}{2}\left(\gamma(1-s)-\frac{s
  \|X \|^2}{\lambda n}\right)\|u-\alpha\|_D^2 + s(\phi^*(-\alpha)+
w^\top X \alpha) ,
\end{align*}
where we used $-u \in \partial \phi(X^\top w)$ which yields
$\phi^*(-u) = - w^\top X u - \phi(X^\top w)$. Therefore
\begin{equation} \label{eqn:PC3}
A-B \ge s\left[\phi(X^\top w) + \phi^*(-\alpha) + w^\top X \alpha +
\left(\frac{\gamma(1-s)}{2} - \frac{s
  \|X\|^2}{2\lambda n}\right) \|u-\alpha\|_D^2 \right] ~.
\end{equation}
Next note that with $w=\nabla g^*(v)$, we have $g(w)+g^*(v)= w^\top v$. Therefore:
\begin{align*} 
P(w)-D(\alpha) &= \frac{1}{n} \sum_{i=1}^n \phi_i(X_i^\top w) +
  \lambda g(w) - \left(-\frac{1}{n} \sum_{i=1}^n
  \phi^*_i(-\alpha_i) - \lambda g^*(v) \right) \\
&= \frac{1}{n} \sum_{i=1}^n \phi_i(X_i^\top w) 
+ \frac{1}{n} \sum_{i=1}^n \phi^*_i(-\alpha_i)  + \lambda w^\top v \\
&= \frac{1}{n} \sum_{i=1}^n \left( \phi_i(X_i^\top w) +
  \phi^*_i(-\alpha_i) 
+ w^\top X_i \alpha_i \right)  .
\end{align*}
Therefore, if we take expectation of \eqref{eqn:PC3} w.r.t. the choice
of $i$ we obtain that
\[
\frac{1}{s}\, \E[A-B] \ge  \E[P(w)-D(\alpha)] - \frac{s}{2\lambda
    n} \cdot \underbrace{\frac{1}{n} \sum_{i=1}^n \left(\|X_i\|^2 -
      \frac{\gamma(1-s)\lambda n}{s}\right) \|u_i-\alpha_i\|_D^2 }_{= G^{(t)}} .
\]
We have obtained that
\begin{equation} \label{eqn:DualSObyGap}
\frac{n}{s}\, \E[D(\alpha^{(t)})-D(\alpha^{(t-1)})] \ge
\E[P(w^{(t-1)})-D(\alpha^{(t-1)})] - \frac{s\,G^{(t)}}{2\lambda n} ~.
\end{equation}
Multiplying both sides by $s/n$ concludes the proof of the lemma.
\end{proof}

We also use the following simple lemma:
\begin{lemma} \label{lem:LBdual}
For all $\alpha$, $D(\alpha) \le P(w^*) \le P(0) \le 1$. In addition, 
$D(0) \ge 0$. 
\end{lemma}
\begin{proof}
The first inequality is by weak duality, the second is by the
optimality of $w^*$, and the third by the assumption that $n^{-1} \sum_i \phi_i(0)
\le 1$. For the last inequality we use
$-\phi^*_i(0) =- \max_z (0-\phi_i(z)) = \min_z \phi_i(z) \ge 0$, 
which yields $D(0) \ge 0$. 
\end{proof}

\subsection{Proof of \thmref{thm:smooth}}
\begin{proof}[Proof of \thmref{thm:smooth}]
The assumption that $\phi_i$ is $(1/\gamma)$-smooth implies that
$\phi_i^*$ is $\gamma$-strongly-convex. 
We will apply \lemref{lem:key} with $s = 
\frac{\lambda n \gamma}{R^2 + \lambda n \gamma } \in [0,1]$. Recall that
$\|X_i\| \le R$. Therefore, 
the choice of $s$ implies that 
\[
\|X_i\|^2 -
      \frac{\gamma(1-s)\lambda n}{s} \le R^2 - \frac{1-s}{s/(\lambda
        n \gamma )} = R^2 - R^2 = 0 ~,
\] and hence $G^{(t)} \le 0$ for
all $t$. This yields, 
\[
\E[D(\alpha^{(t)})-D(\alpha^{(t-1)})] \ge  \frac{s}{n}\,
\E[P(w^{(t-1)})-D(\alpha^{(t-1)})] ~.
\]
But since $\epsilon_D^{(t-1)} := D(\alpha^*)-D(\alpha^{(t-1)}) \le P(w^{(t-1)})-D(\alpha^{(t-1)})$ and $D(\alpha^{(t)})-D(\alpha^{(t-1)})
= \epsilon_D^{(t-1)} - \epsilon_D^{(t)}$, we obtain that 
\[
\E[ \epsilon_D^{(t)} ] \le \left(1 -
  \tfrac{s}{n}\right)\E[\epsilon_D^{(t-1)}] \le \left(1 -
  \tfrac{s}{n}\right)^t \E[\epsilon_D^{(0)}] \le \left(1 -
  \tfrac{s}{n}\right)^t \le \exp(-st/n) = \exp\left(-\frac{\lambda \gamma t}{R^2
    + \lambda \gamma n}\right)~.
\]
This would be smaller than $\epsilon_D$ if 
\[
t \ge \left(n +
  \tfrac{R^2}{\lambda \gamma}\right) \, \log(1/\epsilon_D) ~.
\]
It implies that
\begin{equation}
\E[P(w^{(t)})-D(\alpha^{(t)})]  \le \frac{n}{s}
\E[\epsilon_D^{(t)} - \epsilon_D^{(t+1)}] \le \frac{n}{s} \E[\epsilon_D^{(t)}] . \label{eqn:dgap-bound-smooth}
\end{equation}
So, requiring $\epsilon_D^{(t)} \le \frac{s}{n} \epsilon_P$ we obtain
a duality gap of at most $\epsilon_P$. This means that we should
require
\[
t \ge \left(n +
  \tfrac{R^2}{\lambda \gamma}\right) \, \log( (n + \tfrac{R^2}{\lambda \gamma})   \cdot \tfrac{1}{\epsilon_P}) ~,
\]
which proves the first part of \thmref{thm:smooth}. 

Next, we sum \eqref{eqn:dgap-bound-smooth} over $t=T_0,\ldots,T-1$ to obtain
\[
\E\left[ \frac{1}{T-T_0} \sum_{t=T_0}^{T-1} (P(w^{(t)})-D(\alpha^{(t)}))\right] \le 
\frac{n}{s(T-T_0)} \E[D(\alpha^{(T)})-D(\alpha^{(T_0)})] .
\]
Now, if we choose $\bar{w},\bar{\alpha}$ to be either the average
vectors or a randomly chosen vector over $t \in \{T_0+1,\ldots,T\}$,
then the above implies
\[
\E[ P(\bar{w})-D(\bar{\alpha})] \le 
\frac{n}{s(T-T_0)} \E[D(\alpha^{(T)})-D(\alpha^{(T_0)})] 
\le \frac{n}{s(T-T_0)} \E[\epsilon_D^{(T_0)})] . 
\]
It follows that in order to obtain a result of 
$\E[ P(\bar{w})-D(\bar{\alpha})] \le \epsilon_P$, we only need to have
\[
\E[\epsilon_D^{(T_0)})] \leq \frac{s (T-T_0) \epsilon_P}{n} = \frac{(T-T_0) \epsilon_P}{n + \frac{R^2}{\lambda \gamma}} .
\]
This implies the second part of \thmref{thm:smooth}, and concludes the proof.
\end{proof}

\subsection{Proof of \thmref{thm:Lipschitz}}
Next, we turn to the case of Lipschitz loss function. We rely on the following lemma. 
\begin{lemma} \label{lem:LipConjDom}
Let $\phi : \reals^k \to \reals$ be an $L$-Lipschitz function w.r.t. a
norm $\|\cdot\|_P$ and let $\|\cdot\|_D$ be the dual norm. Then,
for any $\alpha \in \reals^k$ s.t. $\|\alpha\|_D > L$ we have that $\phi^*(\alpha) =
\infty$. 
\end{lemma}
\begin{proof}
  Fix some $\alpha$ with $\|\alpha\|_D > L$. Let $x_0$ be a vector
  such that $\|x_0\|_P = 1$ and $\alpha^\top x_0 = \|\alpha\|_D$ (this is a
  vector that achieves the maximal objective in the definition of the
  dual norm). By definition of
  the conjugate we have
\begin{align*}
\phi^*(\alpha)  &= \sup_x [\alpha^\top\,x - \phi(x)] \\
&\ge -\phi(0) + \sup_{x } [\alpha^\top\,x - (\phi(x) - \phi(0))] \\
&\ge -\phi(0) + \sup_{x } [\alpha^\top\,x - L \|x-0\|_P] \\
&\ge -\phi(0) + \sup_{c > 0 } [\alpha^\top\,(cx_0) - L \|cx_0\|_P] \\
&= -\phi(0) + \sup_{c > 0} (\|\alpha\|_D-L)\,c = \infty ~.
\end{align*}
\end{proof}

A direct corollary of the above lemma is:
\begin{lemma} \label{lem:GboundLip} Suppose that for all $i$, $\phi_i$
  is $L$-Lipschitz w.r.t. $\|\cdot\|_P$. Let $G^{(t)}$ be as defined
  in \lemref{lem:key} (with $\gamma=0$). Then, $G^{(t)} \le
  4\,R^2\,L^2$.
\end{lemma}
\begin{proof}
Using \lemref{lem:LipConjDom} we know that $\|\alpha^{(t-1)}_i\|_D \le L$, and in
addition by the relation of Lipschitz and sub-gradients we have $\|u^{(t-1)}_i\|_D
\le L$. Combining this with the triangle inequality we obtain that $\|u^{(t-1)}_i-\alpha^{(t-1)}_i\|_D^2 \le 4L^2$, and the proof follows. 
\end{proof}

We are now ready to prove \thmref{thm:Lipschitz}.
\begin{proof}[Proof of \thmref{thm:Lipschitz}]
Let $G = \max_t G^{(t)}$ and note that by \lemref{lem:GboundLip} we
have $G \le 4R^2L^2$. \lemref{lem:key}, with $\gamma=0$, tells us that
\begin{equation} \label{eqn:dpeqnL}
\E[D(\alpha^{(t)})-D(\alpha^{(t-1)})] \ge  \frac{s}{n}\,
\E[P(w^{(t-1)})-D(\alpha^{(t-1)})] - \left(\frac{s}{n}\right)^2
\frac{G}{2\lambda} ~,
\end{equation}
which implies that
\[
\E[\epsilon_D^{(t)}] \le \left(1 - \tfrac{s}{n}\right) 
\E[\epsilon_D^{(t-1)}] + \left(\tfrac{s}{n}\right)^2
\tfrac{G}{2\lambda} ~.
\]
We next show that the above yields
\begin{equation} \label{eqn:DualSOL}
\E[\epsilon_D^{(t)}] \le \frac{2 G}{\lambda(2 n + t-t_0)} ~
\end{equation}
for all $t \ge t_0 = \max(0,\lceil n \log(2 \lambda n \epsilon_D^{(0)}/G ) \rceil)$.
Indeed, let us choose $s=1$, then at $t=t_0$, we have
\[
\E[\epsilon_D^{(t)}] \le \left(1 - \tfrac{1}{n}\right)^t \epsilon_D^{(0)} +
\tfrac{G}{2\lambda n^2} \tfrac{1}{1 - (1-1/n)} \le e^{-t/n} \epsilon_D^{(0)} +
\tfrac{G}{2\lambda n}
\le \tfrac{G}{\lambda n} ~.
\]
This implies that \eqref{eqn:DualSOL} holds at $t=t_0$.
For $t > t_0$ we use an inductive argument. 
Suppose the claim holds for $t-1$, therefore
\[
\E[\epsilon_D^{(t)}] \le \left(1 - \tfrac{s}{n}\right) 
\E[\epsilon_D^{(t-1)}] + \left(\tfrac{s}{n}\right)^2
\tfrac{G}{2\lambda} \le 
\left(1 - \tfrac{s}{n}\right) \tfrac{2 G}{\lambda(2n + t -1-t_0)}
 + \left(\tfrac{s}{n}\right)^2
\tfrac{G}{2\lambda} .
\]
Choosing $s = 2n/(2n-t_0+t-1) \in [0,1]$ yields
\begin{align*}
\E[\epsilon_D^{(t)}] &\le
\left(1 - \tfrac{2}{2n-t_0+t-1}\right) \tfrac{2 G}{\lambda(2n-t_0 + t -1)}
 + \left(\tfrac{2}{2n-t_0+t-1}\right)^2
\tfrac{G}{2\lambda} \\
&= \tfrac{2 G}{\lambda(2n-t_0 + t -1)}\left(1 - \tfrac{1}{2n-t_0 + t -1}\right) \\
&= \tfrac{2 G}{\lambda(2n-t_0 + t -1)}\tfrac{2n-t_0+t-2}{2n-t_0 + t -1} \\
&\le \tfrac{2 G}{\lambda(2n-t_0 + t -1)}\tfrac{2n-t_0+t-1}{2n-t_0 + t} \\
&= \tfrac{2 G}{\lambda(2n-t_0 + t)} ~.
\end{align*}
This provides a bound on the dual sub-optimality. We next turn to
bound the duality gap. 
Summing \eqref{eqn:dpeqnL} over $t=T_0+1,\ldots,T$ and rearranging terms we
obtain that
\[
\E\left[ \frac{1}{T-T_0} \sum_{t=T_0+1}^T (P(w^{(t-1)})-D(\alpha^{(t-1)}))\right] \le 
\frac{n}{s(T-T_0)} \E[D(\alpha^{(T)})-D(\alpha^{(T_0)})] + \frac{s\,G}{2\lambda n} 
\]
Now, if we choose $\bar{w},\bar{\alpha}$ to be either the average
vectors or a randomly chosen vector over $t \in \{T_0+1,\ldots,T\}$,
then the above implies
\[
\E[ P(\bar{w})-D(\bar{\alpha})] \le 
\frac{n}{s(T-T_0)} \E[D(\alpha^{(T)})-D(\alpha^{(T_0)})] +
\frac{s\,G}{2\lambda n}  ~.
\]
If $T \ge n+T_0$ and $T_0 \geq t_0$, we can set $s = n/(T-T_0)$ and combining with
\eqref{eqn:DualSOL} we obtain
\begin{align*}
\E[ P(\bar{w})-D(\bar{\alpha})] &\le 
\E[D(\alpha^{(T)})-D(\alpha^{(T_0)})] + \frac{G}{2\lambda (T-T_0)}  \\
&\le \E[D(\alpha^*)-D(\alpha^{(T_0)})]+ \frac{G}{2\lambda (T-T_0)} \\
&\le \frac{2G}{\lambda(2n-t_0+T_0)} + \frac{G}{2\lambda (T-T_0)} ~.
\end{align*}
A sufficient condition for the above to be smaller than $\epsilon_P$
is that $T_0 \ge \frac{4G}{\lambda\epsilon_P} - 2n + t_0$ and $T \ge T_0 +
\frac{G}{\lambda\epsilon_P}$. It also implies that $\E[D(\alpha^*)-D(\alpha^{(T_0)})] \leq \epsilon_P/2$.
Since we also need $T_0 \ge t_0$ and $T-T_0
\ge n$, the overall number of required iterations can be
\[
T_0 \geq \max\{t_0 , 4G/(\lambda \epsilon_P) -2 n + t_0\} ,
\quad T-T_0 \geq \max\{n,G/(\lambda \epsilon_P)\} .
\]
We conclude the proof by noticing that $\epsilon_D^{(0)} \leq 1$ (\lemref{lem:LBdual}), which
implies that $t_0 \leq \max(0,\lceil n \log(2 \lambda n/G ) \rceil)$.
\end{proof}

\bibliographystyle{plainnat}
\bibliography{curRefs}

\end{document}